\documentclass[letterpaper]{article} 
\usepackage{aaai2027}  
\usepackage[hyphens]{url}  
\usepackage{graphicx} 
\usepackage{natbib}  
\usepackage{caption} 
\usepackage{algorithm}
\usepackage{algorithmic}

\usepackage{newfloat}
\usepackage{listings}
\DeclareCaptionStyle{ruled}{labelfont=normalfont,labelsep=colon,strut=off} 
\floatstyle{ruled}
\newfloat{listing}{tb}{lst}{}
\floatname{listing}{Listing}

\usepackage{booktabs}
\usepackage{multirow}
\usepackage{amsmath,amssymb,amsthm,bm}
\usepackage{xspace}

\newcommand{\MSSA}{\operatorname{MSSA}}

\newcommand{\SIGReg}{\mathcal{L}_{\mathrm{SIGReg}}}

\newcommand{\Rcomp}{R^{c}}
\newcommand{\ourmodel}{AoT-ADMM\xspace}

\DeclareMathOperator{\tr}{tr}
\theoremstyle{plain}
\newtheorem{proposition}{Proposition}
\theoremstyle{remark}

\title{Attention-Only White-Box Transformer via LeJEPA-Based Self-Supervised Pretraining}

\author{
    Yang Bai \textsuperscript{\rm 1},
    Linyuan Wang \textsuperscript{\rm 1},
    Haoyang Jiang \textsuperscript{\rm 1},
    Nuolin Sun \textsuperscript{\rm 1},
    Libin Hou \textsuperscript{\rm 1},
    Bin Yan\corresponding \textsuperscript{\rm 1}
}

\affiliations{
    \textsuperscript{\rm 1}Information Engineering University, ZhengZhou,China\\
    yangbai\_111111@163.com, ybspace@hotmail.com
}

\begin{document}

\maketitle

\begin{abstract}
Existing studies on self-supervised learning for white-box networks typically decouple the derivation of white-box networks via optimization algorithms from self-supervised learning paradigms. In this work, we instead revisit the two components from a joint perspective. The LeJEPA-based self-supervised framework assumes an isotropic Gaussian distribution as the optimal embedding distribution for downstream tasks, which is conceptually equivalent to the expansion term $R(Z)$ in the sparse rate reduction objective guiding white-box Transformer optimization. Building on this observation, we use the LeJEPA self-supervised paradigm to optimize $R(Z)$, and derive the remaining terms $R^{c}(Z\mid U_{[K]})+\lambda\lVert Z\rVert_{0}$ via the alternating direction method of multipliers (ADMM) into an attention-only Transformer that dispenses with the ISTA structure or MLP layers of the original design. Experimental results demonstrate that our attention-only white-box Transformer achieves classification accuracies of $88.88\%$ on CIFAR-10 and $63.54\%$ on CIFAR-100 at the Base scale under the LeJEPA self-supervised paradigm, while the original white-box Transformer CRATE achieves classification accuracies of $89.18\%$ on CIFAR-10 and $63.56\%$ on CIFAR-100. Our model achieves competitive performance while reducing the parameter count by roughly $31\%$. Beyond the white-box setting, we further investigate standard ViTs and find that replacing all MLP blocks with ReLU activations under knowledge distillation removes approximately 66\% of the parameters while preserving competitive accuracy, motivating further investigation into the potential redundancy of MLP modules in standard ViT architectures.
\end{abstract}


\section{Introduction}
Transformers have become a foundational architecture in modern deep learning
and have achieved remarkable performance in computer vision, natural language processing, and other fields \citep{vaswani2017attention,dosovitskiy2021vit}.
However, the design of the Transformer architecture and its many variants remains largely empirical and lacks a rigorous mathematical interpretation \citep{he2024simplifying,dong2021attention,caron2021dino,chen2021mocov3}. This has largely hindered the development of new Transformer variants with improved efficiency or interpretability. 

In recent years, white-box Transformer architectures, exemplified by the Coding Rate Transformer (CRATE), have offered a promising new perspective on addressing the aforementioned problem. Grounded in the principle of sparse rate reduction, CRATE characterizes an ideal representation through the coordinated interplay among an expansion term \(R(Z)\), a compression term \(R^{c}(Z\mid U_{[K]})\), and a sparsity-inducing term \(\lambda\lVert Z\rVert_0\) \citep{yu2023crate}. In the resulting unrolled architecture, the compression term is approximated by a single gradient-descent step, giving rise to a Multi-Head Subspace Self-Attention (MSSA) module. Meanwhile, the expansion and sparsity terms are integrated through the introduction of an orthogonal sparse dictionary \(D\) and jointly optimized using an iterative shrinkage-thresholding algorithm (ISTA) module. This unrolling framework explicitly interprets the computations performed by the core network components as iterative steps for solving well-defined objective subproblems, thereby providing the Transformer with a clear and unified mathematical interpretation.

Building on this foundation, white-box Transformer architectures have continued to evolve in recent years. CRATE-\(\alpha\) improves both the sparse coding module and the training strategy, substantially enhancing the model’s scalability to large-scale vision tasks \citep{yang2024cratealpha}. The Token Statistics Transformer (ToST), meanwhile, derives an interpretable attention mechanism with linear time and memory complexity from a variational formulation of the maximal coding rate reduction objective \citep{wu2024tost}. In parallel, this line of research has been extended to a diverse range of tasks, including unsupervised object segmentation, language modeling, and long-sequence modeling, highlighting the broad potential of white-box networks in terms of interpretability, computational efficiency, and cross-domain applicability \citep{yu2023segmentation,pai2024cratemae}.

White-box networks seek to uncover the low-dimensional subspace structure underlying token representations directly from data. Because coding-rate objectives can be formulated without semantic labels, they provide a possible route to self-supervised representation learning \citep{yu2020mcr2}. Several recent studies have attempted to bridge this gap. For instance, CRATE-MAE incorporates white-box encoders and decoders into a masked autoencoding framework \citep{pai2024cratemae,he2022mae}, while EMP-SSL introduces the total coding rate \(R(Z)\) into a self-supervised objective to promote representation expansion \citep{tong2023empssl}. However, these approaches remain subject to notable limitations: the former relies on a reconstruction loss that lacks a clear theoretical interpretation, whereas the latter retains a black-box network architecture. These limitations motivate a fundamental question: how can models be trained under a self-supervised paradigm while preserving theoretical self-consistency between the white-box architecture and its learning objective?

To address this question, we revisit the three components of the white-box
objective: the expansion term is transferred to and optimized through the
LeJEPA self-supervised objective, while the compression and sparsity terms are
optimized using the alternating direction method of multipliers (ADMM) \citep{boyd2011admm,doi:10.1137/090777761} and unrolled into an attention-only Transformer architecture. Our approach builds on the close correspondence, in representation space, between the expansion term in the white-box objective and the isotropy constraint imposed by LeJEPA \citep{balestriero2025lejepa,assran2023ijepa}. Based on this correspondence, we formulate a new optimization objective by replacing the expansion term \(R(Z)\) with the LeJEPA objective, while retaining the subspace compression term \(R^{c}(Z\mid U_{[K]})\) and the sparsity term \(\lambda\lVert Z\rVert_0\). We then optimize the resulting objective using ADMM. Unrolling the resulting ADMM iterations yields a white-box self-supervised architecture consisting solely of attention modules. We refer to the resulting Attention-Only Transformer derived from ADMM as \ourmodel. The overall self-supervised training framework of \ourmodel is illustrated in Figure~\ref{fig:architecture}.

Our main contributions are summarized as follows:
\begin{itemize}
    \item We propose an attention-only white-box Transformer based on the LeJEPA self-supervised objective. The model uses the LeJEPA self-supervised learning paradigm to optimize the expansion term $R(Z)$ in maximal sparse rate reduction, while the remaining compression and sparsity terms follow the white-box Transformer derivation paradigm, thereby combining self-supervised learning with a white-box optimization objective.
    \item We optimize the expansion term $R(Z)$ with the LeJEPA self-supervised learning paradigm and optimize the remaining white-box objective terms using ADMM. The compression term corresponds to the attention structure of the Transformer, while the sparsity term admits a closed-form solution through a ReLU operator, ultimately yielding an attention-only white-box Transformer.

    \item Experimental results show that \ourmodel achieves performance comparable to CRATE while reducing parameters by approximately $31\%$. At the Base scale, our model reaches $88.88\%$ on CIFAR-10 and $63.54\%$ on CIFAR-100, compared with CRATE's $89.18\%$ and $63.56\%$. It also consistently outperforms AoT at similar parameter counts. On standard ViTs, replacing all MLP blocks with ReLU activations under knowledge distillation removes approximately $66\%$ of the parameters while preserving competitive accuracy, further demonstrating the potential redundancy of MLP modules.

\end{itemize}

\section{Related Work}
\subsection{White-Box Self-Supervised Learning}
Conventional Transformer modules are typically selected through empirical trial and error \citep{he2024simplifying,dong2021attention}, without a strict mathematical correspondence between the architecture and its learning objective. CRATE provides an alternative by deriving its basic operators from sparse rate reduction, such that each network operation has an explicit optimization interpretation \citep{yu2023crate}. This theoretical property makes white-box models naturally suited to self-supervised representation learning. Nevertheless, previous attempts have not established a consistent theoretical source for both the training objective and the network architecture. CRATE-MAE adopts a white-box encoder but still relies on an uninterpretable reconstruction loss \citep{pai2024cratemae}, whereas EMP-SSL introduces the total coding rate $R(Z)$ but continues to employ a black-box architecture \citep{tong2023empssl}. In contrast, we first establish the theoretical connection between an isotropic Gaussian prior and the coding-rate expansion term $R(Z)$ and use it as the training objective, and then derive the complete network architecture by solving the remaining white-box objective terms. To our knowledge, this is the first self-supervised formulation in which both the training loss and the forward computation originate from a unified white-box theoretical framework.

\subsection{Attention-Only Transformers}
Transformer simplification studies have shown that some block components can be removed or rearranged \citep{he2024simplifying}. However, naively removing the feed-forward branch can weaken representation diversity, and theoretical analysis has identified rank-collapse behavior in pure-attention stacks \citep{dong2021attention}. An interesting study by Wang et al. provides such a derivation:
they model tokens as a noisy mixture of low-rank Gaussians, treat representation learning as compressing noisy tokens back onto their subspaces, read multi-head subspace self-attention as an iterative denoising operator, and unroll that operator into a white-box Attention-Only Transformer (AoT) built from attention and residual connections alone \citep{wang2025aot}. Their target is subspace denoising, so the
attention-only form follows from recovering clean representations rather than
from optimizing a full representation-learning objective.
In contrast, our layer is derived from a different optimization problem: both its attention-only architecture and its training procedure emerge naturally from the white-box objective. Specifically, we preserve the sparse rate reduction objective, move its expansion term into the training loss, and obtain the attention-only architecture by unrolling the solution to the remaining compression and sparsity terms.

\begin{figure*}[htbp]
\centering
\includegraphics[width=0.88\textwidth]{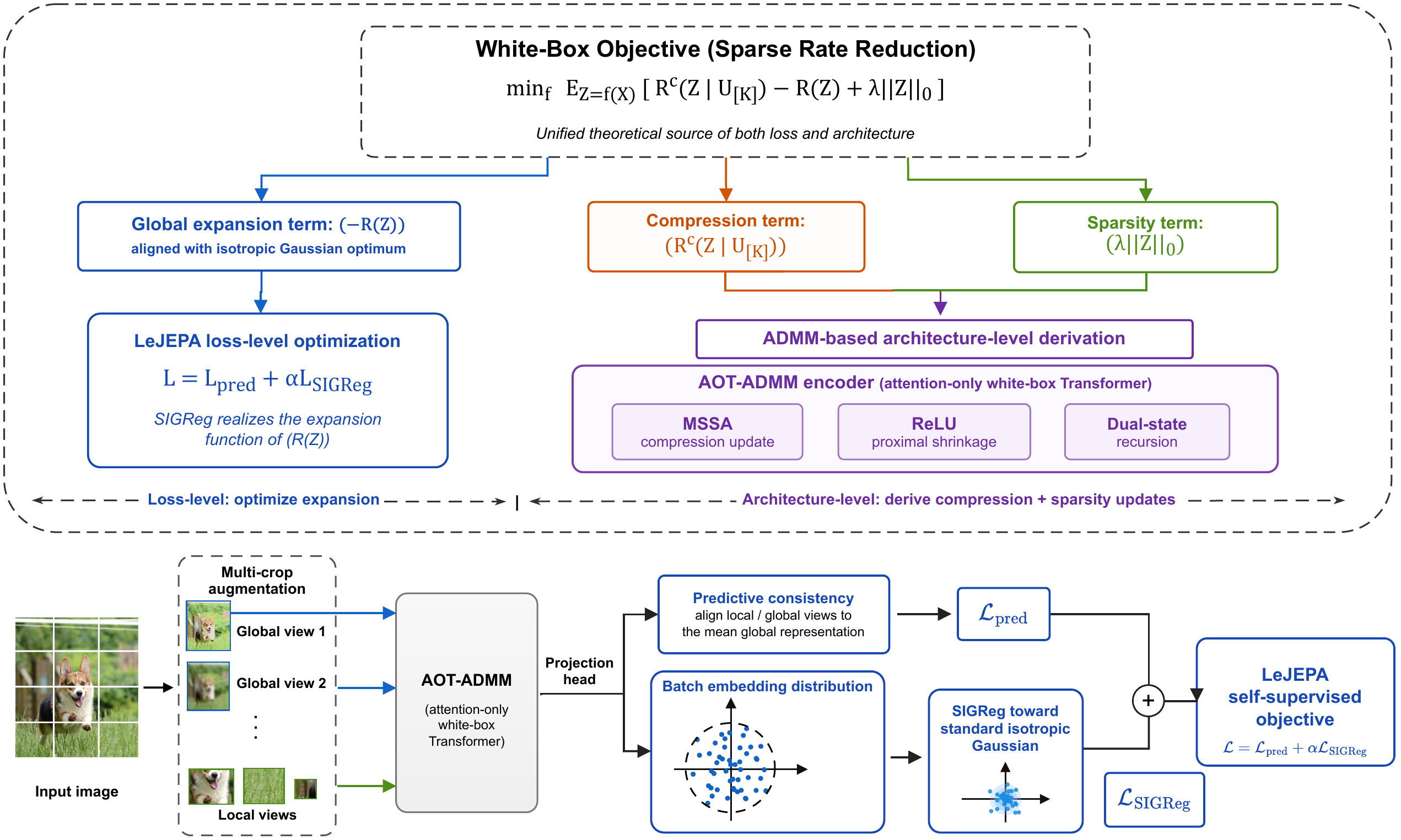}
\caption{Overview of the proposed white-box self-supervised framework. LeJEPA optimizes the global expansion term, while ADMM unrolls the compression and sparsity terms into the attention-only AoT-ADMM encoder, deriving both the training objective and architecture from a unified white-box objective.
}
\label{fig:architecture}
\end{figure*}

\section{Method}

This section presents the complete derivation of an attention-only white-box Transformer under LeJEPA-based self-supervised pretraining. As illustrated in Figure~\ref{fig:architecture}, the proposed framework derives both the training objective and the network architecture from the sparse rate reduction principle. We first establish the connection between Sketched Isotropic Gaussian Regularization (SIGReg) in LeJEPA and the global expansion term \(R(Z)\), which enables the original objective to be separated into loss-level and architecture-level components. We then apply ADMM to the remaining compression and sparsity terms, yielding a three-state recursion. Finally, these ADMM iterations are unrolled along network depth to construct the attention-only forward architecture.

\subsection{Optimizing the Global Expansion Term}

Following maximal coding-rate reduction and its sparse white-box Transformer formulation \citep{yu2020mcr2,yu2023crate}, sparse rate reduction characterizes a desirable representation as globally expanded, compact within low-dimensional subspaces, and sparse. Let
$Z=f(X)\in\mathbb{R}^{d\times N}$ denote $N$ token representations and let
$U_{[K]}=(U_k)_{k=1}^{K}$ denote a collection of low-dimensional subspace
bases. The corresponding objective is
\begin{equation}
\min_{f\in\mathcal{F}}\;
\mathbb{E}_{Z=f(X)}
\left[
\Rcomp(Z\mid U_{[K]})
-R(Z)
+\lambda\lVert Z\rVert_0
\right],
\label{eq:srr}
\end{equation}
where the global coding rate is defined as
\begin{equation}
R(Z) =
\frac{1}{2}
\log\det\left(
I_d+\frac{d}{N\epsilon^2}ZZ^\top
\right).
\label{eq:global_rate}
\end{equation}
By maximizing the log-determinant of the representation covariance, $R(Z)$
encourages the representations to occupy a large volume in the ambient space
and thereby prevents global collapse. Equivalently, under a fixed
representation energy, maximizing $R(Z)$ favors a balanced covariance
spectrum and discourages the representations from concentrating along only a
few directions. Total coding-rate maximization has accordingly been
interpreted as a soft covariance regularization mechanism closely related to
covariance-based self-supervised objectives such as VICReg
\citep{tong2023empssl}.

SIGReg similarly
encourages the embedding distribution to remain globally expanded by driving
it toward a standard isotropic Gaussian. In practice, SIGReg is evaluated on
the embeddings used by the LeJEPA objective, which we denote by $Z'$ to
distinguish them from the token matrix $Z$ in
Eq.~\eqref{eq:srr}. The two variables are produced within the same
representation-learning pipeline, but need not denote the same finite-sample
matrix. Our correspondence concerns the expansion effect that the two
objectives impose on the learned representation distribution.

SIGReg measures the discrepancy between random one-dimensional projections
of the embedding distribution and a standard Gaussian using an Epps--Pulley
characteristic-function statistic \citep{balestriero2025lejepa}. Around the
isotropic Gaussian optimum, the discrepancy for centered embeddings can be
decomposed as
\begin{equation}
\SIGReg(Z')
=
C_w\,
\mathbb{E}_{a\sim\mathrm{Unif}(\mathbb{S}^{d-1})}
\left[
\bigl(a^\top(\Sigma-I_d)a\bigr)^2
\right]
+
\mathcal{L}_{>2},
\label{eq:sigreg}
\end{equation}
where $C_w>0$, $\Sigma$ denotes the covariance of $Z'$, and
$\mathcal{L}_{>2}$ collects deviations in skewness, kurtosis, and
higher-order statistics. We denote the covariance-level component in
Eq.~\eqref{eq:sigreg} by
$\mathcal{L}_{\mathrm{SIGReg}}^{(2)}$.

\paragraph{Theorem 1 (Consistency of the expansion objectives).}
For centered embeddings with fixed normalized energy
$\operatorname{tr}(\Sigma)=d$, the coding-rate objective and the covariance
component of SIGReg share the same isotropic optimum:
\begin{equation}
\arg\max_{\substack{\Sigma\succeq 0\\
                    \operatorname{tr}(\Sigma)=d}}
R(\Sigma)
=
\arg\min_{\substack{\Sigma\succeq 0\\
                    \operatorname{tr}(\Sigma)=d}}
\mathcal{L}_{\mathrm{SIGReg}}^{(2)}(\Sigma)
=
\{I_d\}.
\label{eq:theorem}
\end{equation}
Moreover, around $\Sigma=I_d$, the coding-rate gap
$R(I_d)-R(\Sigma)$ and
$\mathcal{L}_{\mathrm{SIGReg}}^{(2)}(\Sigma)$ are both proportional, to
leading order, to $\lVert\Sigma-I_d\rVert_F^2$. The complete derivation is provided in Appendix~\ref{app:theorem1}.

Theorem~1 establishes that coding-rate maximization and SIGReg implement a
consistent expansion principle at the covariance level: both favor an
isotropic representation covariance and penalize directional concentration
around the optimum. They are nevertheless not globally identical objectives.
The coding rate depends only on second-order statistics, whereas SIGReg
additionally constrains higher-order deviations through
$\mathcal{L}_{>2}$. Thus, SIGReg preserves the covariance-expansion effect of
$R(Z)$ while strengthening it toward a full-distribution isotropic Gaussian
constraint.

This correspondence allows the global expansion effect in
Eq.~\eqref{eq:srr} to be incorporated into the LeJEPA training objective.
The compression and sparsity terms can then be retained explicitly for the
subsequent derivation of the network architecture.

\subsection{Optimizing the Compression and Sparsity Terms}

With the global expansion term handled by the training loss, the remaining
optimization problem is
\begin{equation}
\min_Z\;
\Rcomp(Z\mid U_{[K]})
+\lambda\lVert Z\rVert_0.
\label{eq:compression_sparsity}
\end{equation}
For $U_k\in\mathbb{R}^{d\times p}$ with $U_k^\top U_k=I_p$, the coding rate of
the representations projected onto the $k$-th subspace is defined as
\begin{equation}
R(U_k^\top Z) =
\frac{1}{2}
\log\det\left(
I_p+
\frac{p}{N\epsilon^2}
U_k^\top ZZ^\top U_k
\right),
\label{eq:subspace_rate}
\end{equation}
and the corresponding compression rate is
\begin{equation}
\Rcomp(Z\mid U_{[K]}) =
\sum_{k=1}^{K}R(U_k^\top Z).
\label{eq:compression_rate}
\end{equation}

To obtain a tractable sparsity surrogate, we replace the $\ell_0$ penalty by an $\ell_1$ penalty, as commonly done in sparse recovery \citep{chen1998basis}, and introduce an auxiliary variable $V$:
\begin{equation}
\min_{Z,V}\;
\Rcomp(Z\mid U_{[K]})
+\lambda\lVert V\rVert_1,
\quad
\mathrm{s.t.}\; Z=V.
\label{eq:split}
\end{equation}
Using the scaled-form ADMM with dual variable $W$ and penalty parameter $\rho>0$ \citep{boyd2011admm,eckstein1992douglas}, the augmented
Lagrangian, up to a term independent of $Z$ and $V$, is
\begin{equation}
\mathcal{L}_{\rho}(Z,V,W)
=
\Rcomp(Z\mid U_{[K]})
+\lambda\lVert V\rVert_1
+\frac{\rho}{2}\lVert Z-V+W\rVert_F^2.
\label{eq:auglag}
\end{equation}

The update of $Z$-step is given by
\begin{equation}
Z^{t+1}
=
\arg\min_Z\;
\Rcomp(Z\mid U_{[K]})
+\frac{\rho}{2}
\lVert Z-V^t+W^t\rVert_F^2.
\label{eq:zsub}
\end{equation}
Following CRATE, we approximate this subproblem with a single linearized gradient step \citep{yu2023crate,ouyang2015linearized}. In particular, the gradient of the compression rate is approximated as
\begin{equation}
\nabla_Z\Rcomp(Z\mid U_{[K]})
\approx
\gamma
\left[
Z-\MSSA(Z\mid U_{[K]})
\right],
\label{eq:compression_gradient}
\end{equation}
where $\gamma>0$ is determined by the coding precision, and multi-head
subspace self-attention is defined as
\begin{equation}
\MSSA(Z\mid U_{[K]}) =
\sum_{k=1}^{K}
U_kU_k^\top Z\,
\operatorname{softmax}\left(
Z^\top U_kU_k^\top Z
\right),
\label{eq:mssa}
\end{equation}

Combining Eq.~\eqref{eq:compression_gradient} with the quadratic penalty, the
gradient of the $Z$ subproblem is approximated by
\begin{equation}
\gamma
\left[
Z-\MSSA(Z\mid U_{[K]})
\right]
+\rho\left(
Z-V^t+W^t
\right).
\label{eq:zsub_gradient}
\end{equation}
A gradient step with step size $\eta>0$ therefore yields
\begin{align}
Z^{t+1}
={}&
(1-\eta\gamma-\eta\rho)Z^t
+\eta\gamma
\MSSA(Z^t\mid U_{[K]}^t)
\nonumber\\
&+\eta\rho(V^t-W^t).
\label{eq:zupdate}
\end{align}

The update of $V$-step is given by
\begin{equation}
V^{t+1}
=
\arg\min_V\;
\lambda\lVert V\rVert_1
+\frac{\rho}{2}
\lVert Z^{t+1}-V+W^t\rVert_F^2,
\end{equation}
whose solution is the proximal operator of the $\ell_1$ norm \citep{daubechies2004iterative,beck2009fista,parikh2014proximal},
\begin{equation}
V^{t+1}
=
\operatorname{shrink}_{\lambda/\rho}
(
Z^{t+1}+W^t
).
\end{equation}

Following CRATE, we further impose a non-negativity constraint on the intermediate representations. Under this constraint, the proximal mapping reduces to the one-sided soft-thresholding operator, yielding

\begin{equation}
V^{t+1}
=
\operatorname{ReLU}
(
Z^{t+1}+W^t-\lambda/\rho
).
\end{equation}

This update preserves the sparse proximal interpretation while remaining compatible with the nonlinear activation commonly adopted in transformer architectures.

\subsection{Attention-Only White-Box Transformer}

\begin{figure*}[htbp]
\centering
\includegraphics[width=15cm, height=8cm, keepaspectratio=false]{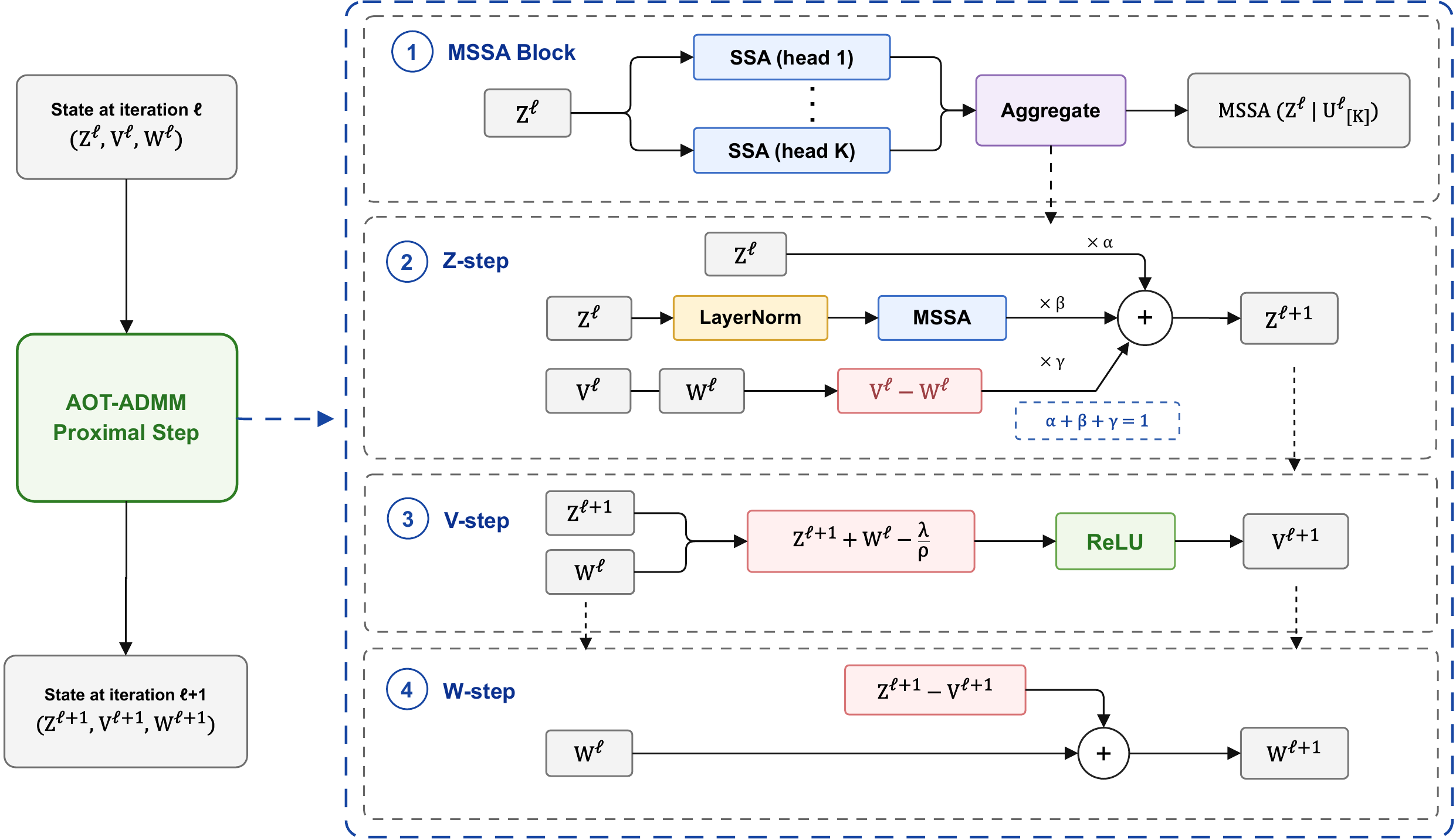}
\caption{One AoT-ADMM layer as an unrolled ADMM iteration. MSSA
realizes the compression update, the ReLU proximal operator implements
nonnegative shrinkage, and the scaled multiplier state tracks the consistency
between $Z$ and $V$. The three branch coefficients are initialized from the
ADMM-derived update and adapt dynamically during training.}
\label{fig:onelayer}
\end{figure*}

Following the algorithm-unrolling paradigm exemplified by learned sparse-coding iterations \citep{gregor2010lista}, mapping the ADMM iteration index $t$ to the layer index $\ell$ converts the
iterative updates into the forward computation of one network layer:
\begin{equation}
\left\{
\begin{aligned}[c]
Z^{\ell+1}
={}&
(1-\eta\gamma-\eta\rho)Z^\ell
+\eta\gamma\MSSA(Z^\ell\mid U_{[K]}^\ell) \\
&+\eta\rho(V^\ell-W^\ell), \\[2pt]
V^{\ell+1}
={}&
\operatorname{ReLU}
\left(
Z^{\ell+1}+W^\ell-\lambda/\rho
\right), \\[2pt]
W^{\ell+1}
={}&
W^\ell+Z^{\ell+1}-V^{\ell+1}.
\end{aligned}
\right.
\label{eq:layer_updates}
\end{equation}

The resulting network propagates the state triplet
$(Z^\ell,V^\ell,W^\ell)$ through depth, with each layer corresponding to one
unrolled ADMM iteration. As illustrated in Figure~\ref{fig:onelayer}, the
$Z$-step performs MSSA-based compression with an ADMM consistency correction,
the $V$-step applies nonnegative shrinkage through a ReLU proximal operator,
and the $W$-step updates the dual state to promote agreement between
$Z^{\ell+1}$ and $V^{\ell+1}$. All layers share the same ADMM-derived update
form, while using independently parameterized subspace bases
$U_{[K]}^\ell$ to adapt the compression model to the representation
distribution at each depth. Stacking $L$ such iterations yields the complete
attention-only encoder.

In implementation, the three branch coefficients in
Eq.~\eqref{eq:layer_updates} are initialized from the ADMM-derived update and
subsequently learned under a nonnegative unit-sum constraint. Figure~\ref{fig:onelayer} therefore
summarizes both the optimization interpretation and the forward computation of
a single \ourmodel layer.

We embed the resulting encoder in LeJEPA. Multi-crop augmentations \citep{caron2021dino} generate
global and local views of each image, which are mapped by the shared encoder
and projection head to their corresponding embeddings. A predictive
consistency loss aligns the local and global embeddings with the mean global
representation, while SIGReg regularizes the batch embedding distribution
toward a standard isotropic Gaussian. The total training objective is
\begin{equation}
\mathcal{L}
=
\mathcal{L}_{\text{pred}} + \alpha \, \mathcal{L}_{\text{SIGReg}},
\label{eq:total}
\end{equation}
where \( \alpha \) balances semantic invariance and distributional expansion.
The encoder, projection head, and layerwise subspace bases are optimized
jointly end to end.

\section{Experiments}
\subsection{Experimental Setup}
We conduct experiments on CIFAR-10, CIFAR-100
\citep{krizhevsky2009cifar}, and ImageNet-1K
\citep{deng2009imagenet}. For in-domain self-supervised learning, CRATE and
\ourmodel are pretrained on CIFAR-10 and CIFAR-100 using the same LeJEPA
objective and evaluated by linear probing with frozen backbones, following the standard self-supervised evaluation protocol \citep{chen2020simclr,grill2020byol}. We report the top-1 accuracy, parameters and flops as indicators of representation quality and model
complexity. We further pretrain \ourmodel on ImageNet-1K and fully fine-tune the resulting
backbone on CIFAR-10 and CIFAR-100 to evaluate large-scale pretraining and
transfer. All experiments are conducted on a single NVIDIA A100 (80GB) GPU; detailed software and hardware configurations are provided in the supplementary material.

For the direct comparison with existing attention-only architectures,
\ourmodel and Wang et al.'s AoT are evaluated at
comparable model scales under the same LeJEPA training protocol.

To further examine MLP redundancy beyond the white-box CRATE architecture, we
construct a standard ViT control by retaining its patch embedding,
self-attention, normalization layers, and classification head, while replacing
the MLP in every Transformer block with a ReLU activation. The resulting model
is trained with knowledge distillation and compared with the original ViT in
terms of Top-1 accuracy and parameter count.

For CIFAR pretraining, global and local views are $32\times32$ and $16\times16$, the patch size is 8, and each image produces two global and six local views. We train for 800 epochs with AdamW \citep{loshchilov2019adamw}, a learning rate of $5\times10^{-4}$, weight decay $5\times10^{-2}$, batch size 256, and cosine decay. For ImageNet-1K, global and local views are $128\times128$ and $64\times64$, patch size is 16, and each image produces two global and two local views. We train for 200 epochs with learning rate $5\times10^{-4}$ and batch size 1024; the resulting backbone is then fully fine-tuned on CIFAR.

\subsection{Main Results} We first evaluate in-domain linear probing against CRATE under exactly the same LeJEPA pretraining settings, with results summarized in Table~\ref{tab:cifar}. Across Tiny, Small, and Base, \ourmodel reduces the parameter count by approximately $31\%$ while maintaining competitive performance. At Tiny scale, the accuracy is only $0.10$ points lower on CIFAR-100. At Small scale, \ourmodel improves CIFAR-100 accuracy from $61.64\%$ to $62.41\%$, a gain of $0.77$ points. At Base scale, the gaps narrow to $0.30$ and $0.02$ points on CIFAR-10 and CIFAR-100, respectively.

\begin{table}[t]
\centering
{\setlength{\tabcolsep}{1mm}%
\begin{tabular}{llccc}
\toprule
Scale & Model & FLOPs & Params & Acc. (\%) \\
\midrule
\multicolumn{5}{l}{\textit{CIFAR-10}} \\
\multirow{2}{*}{Tiny}  & CRATE     & 0.31G & 5.41M  & 85.61 \\
                       & \ourmodel & 0.13G & 3.64M  & 84.22 \\
\cmidrule{1-5}
\multirow{2}{*}{Small} & CRATE     & 0.69G & 12.10M & 87.89 \\
                       & \ourmodel & 0.28G & 8.11M  & 87.23 \\
\cmidrule{1-5}
\multirow{2}{*}{Base}  & CRATE     & 1.22G & 21.44M & 89.18 \\
                       & \ourmodel & 0.50G & 14.36M & 88.88 \\
\midrule
\multicolumn{5}{l}{\textit{CIFAR-100}} \\
\multirow{2}{*}{Tiny}  & CRATE     & 0.31G & 5.41M  & 57.48 \\
                       & \ourmodel & 0.13G & 3.64M  & 57.38 \\
\cmidrule{1-5}
\multirow{2}{*}{Small} & CRATE     & 0.69G & 12.10M & 61.64 \\
                       & \ourmodel & 0.28G & 8.11M  & 62.41 \\
\cmidrule{1-5}
\multirow{2}{*}{Base}  & CRATE     & 1.22G & 21.44M & 63.56 \\
                       & \ourmodel & 0.50G & 14.36M & 63.54 \\
\bottomrule
\end{tabular}%
}
\caption{Linear-probe Top-1 accuracy, FLOPs, and parameters under LeJEPA self-supervised pretraining.}
\label{tab:cifar}
\end{table}

Figure~\ref{fig:training_dynamics} further compares the LeJEPA pretraining dynamics of CRATE and \ourmodel. Under identical training settings, the two models follow closely matched convergence trajectories and reach comparable final loss values after 800 epochs, indicating that removing the learnable ISTA dictionary does not adversely affect optimization or convergence.

These results show that the ADMM-derived attention-only structure preserves the representation quality of CRATE while removing the learnable ISTA dictionary and substantially reducing parameters. A consistent reduction in model parameters is observed across all scales: \ourmodel removes 1.77M (3.64M vs. 5.41M), 3.99M (8.11M vs. 12.10M), and 7.08M (14.36M vs. 21.44M) parameters from Tiny, Small, and Base, respectively, corresponding to reductions of $31.1\%$, $31.2\%$, and $31.3\%$. Together with the near-identical Base-scale performance and closely matched pretraining dynamics, this indicates that the expansion function can be transferred to the self-supervised loss without requiring the original ISTA dictionary in every layer.

\begin{figure}[t]
\centering
\includegraphics[width=0.9\columnwidth]{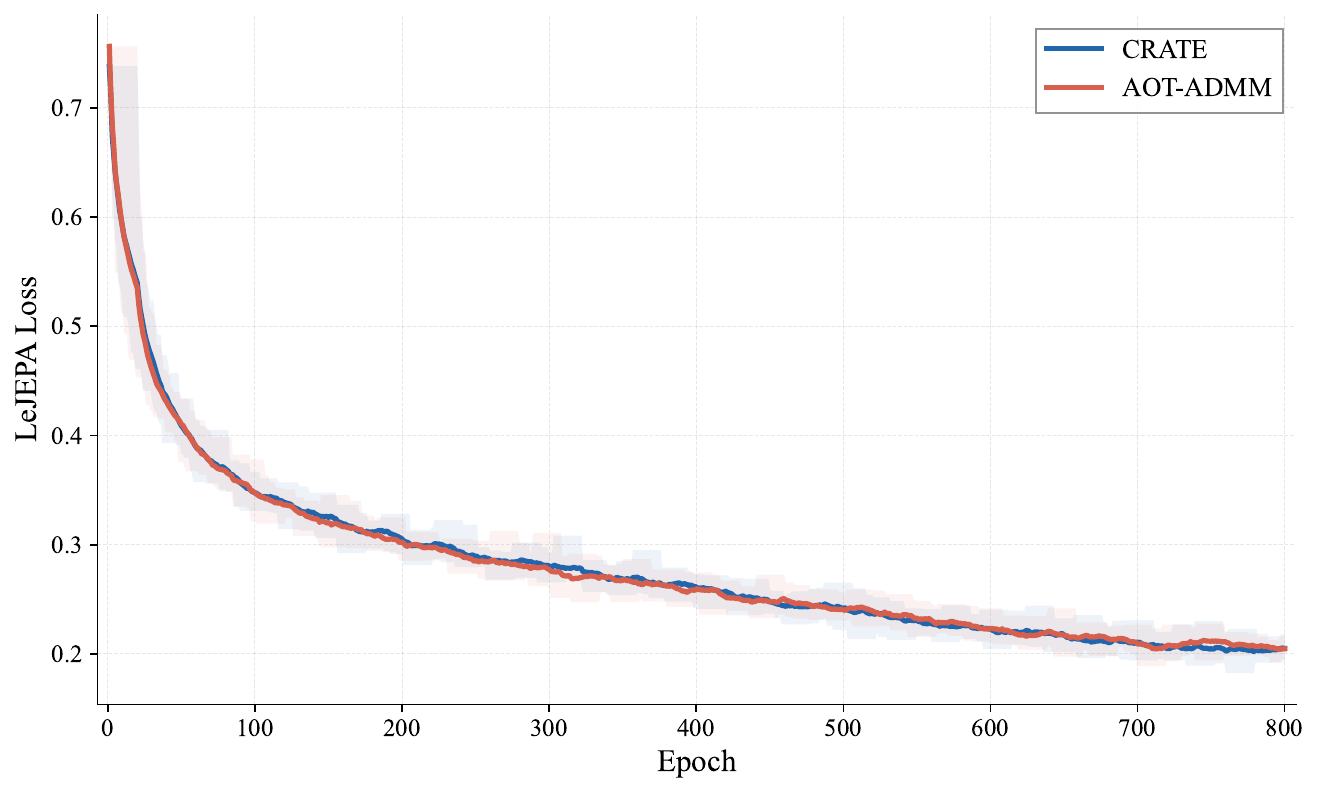}
\caption{Comparison of LeJEPA pretraining loss between CRATE and \ourmodel.
Under identical settings, the two models exhibit closely matched convergence
trajectories and reach comparable final loss values after 800 epochs.}
\label{fig:training_dynamics}
\end{figure}

\begin{table}[t]
\centering
{\setlength{\tabcolsep}{1mm}%
\begin{tabular}{lccc}
\toprule
Model & Params & Pretrain & Acc. (\%) \\
\midrule
\multicolumn{4}{l}{\textit{CIFAR-10}} \\
CRATE     & 21.92M & IN-1K & 96.01 \\
\ourmodel & 14.84M & IN-1K & 95.25 \\
\midrule
\multicolumn{4}{l}{\textit{CIFAR-100}} \\
CRATE     & 21.92M & IN-1K & 81.04 \\
\ourmodel & 14.84M & IN-1K & 80.26 \\
\bottomrule
\end{tabular}%
}
\caption{ImageNet-1K self-supervised pretraining (200 epochs) followed by full CIFAR fine-tuning.}
\label{tab:transfer}
\end{table}

We further assess the transferability of \ourmodel{} under large-scale self-supervised pretraining by pretraining on ImageNet-1K and subsequently fine-tuning on CIFAR-10 and CIFAR-100. As shown in Table~\ref{tab:transfer}, \ourmodel{} achieves $95.25\%$ accuracy on CIFAR-10 and $80.26\%$ on CIFAR-100 with 14.84M parameters, compared with $96.01\%$ and $81.04\%$ for CRATE with 21.92M parameters. These results demonstrate that \ourmodel{} maintains competitive transfer performance while using approximately $32\%$ fewer parameters, highlighting its improved parameter efficiency under large-scale self-supervised pretraining.

\subsection{Comparison with AoT}

Wang et al.'s AoT derives an attention-only architecture by interpreting
representation learning as iterative denoising toward a mixture of
low-dimensional subspaces \citep{wang2025aot}. In contrast, \ourmodel is
derived by combining LeJEPA-based global expansion with ADMM optimization of
the remaining compression and sparsity terms. To isolate the effect of the
architectural formulation, we compare the two models at nearly identical
parameter counts under the same LeJEPA training protocol.

As shown in Table~\ref{tab:aot_compare}, \ourmodel consistently outperforms Wang et al.'s AoT across Tiny, Small, and Base configurations on both CIFAR-10 and CIFAR-100, with nearly identical parameter counts. At the Tiny scale, our model achieves accuracy improvements over AoT of 4.91$\%$ (84.22$\%$ vs.\ 79.31$\%$) on CIFAR-10 and 6.55$\%$ (57.38$\%$ vs.\ 50.83$\%$) on CIFAR-100. At the Small scale, the accuracy gains are 3.87$\%$ (87.23$\%$ vs.\ 83.36$\%$) on CIFAR-10 and 5.90$\%$ (62.41$\%$ vs.\ 56.51$\%$) on CIFAR-100. At the Base scale, our model outperforms AoT by 4.68$\%$ (88.88$\%$ vs.\ 84.20$\%$) in accuracy on CIFAR-10 and by 6.06$\%$ (63.54$\%$ vs.\ 57.48$\%$) on CIFAR-100. These results indicate that removing the MLP alone is not sufficient to obtain an effective attention-only Transformer; the optimization principle underlying the architectural derivation also plays an important role.

\begin{table}[t]
\centering
{\setlength{\tabcolsep}{1mm}%
\begin{tabular}{llrc}
\toprule
Scale & Model & Params & Acc. (\%) \\
\midrule
\multicolumn{4}{l}{\textit{CIFAR-10}} \\
\multirow{2}{*}{Tiny}
  & AoT       & 3.63M  & 79.31 \\
  & \ourmodel & 3.64M  & 84.22 \\
\cmidrule{1-4}
\multirow{2}{*}{Small}
  & AoT       & 8.11M  & 83.36 \\
  & \ourmodel & 8.11M  & 87.23 \\
\cmidrule{1-4}
\multirow{2}{*}{Base}
  & AoT       & 14.35M & 84.20 \\
  & \ourmodel & 14.36M & 88.88 \\
\midrule
\multicolumn{4}{l}{\textit{CIFAR-100}} \\
\multirow{2}{*}{Tiny}
  & AoT       & 3.63M  & 50.83 \\
  & \ourmodel & 3.64M  & 57.38 \\
\cmidrule{1-4}
\multirow{2}{*}{Small}
  & AoT       & 8.11M  & 56.51 \\
  & \ourmodel & 8.11M  & 62.41 \\
\cmidrule{1-4}
\multirow{2}{*}{Base}
  & AoT       & 14.35M & 57.48 \\
  & \ourmodel & 14.36M & 63.54 \\
\bottomrule
\end{tabular}%
}
\caption{Comparison with Wang et al.'s AoT under the LeJEPA evaluation.}
\label{tab:aot_compare}
\end{table}

\subsection{Black-Box Transformer Validation}
We further test whether the proximal update can remove MLP blocks from a standard Vision Transformer \citep{dosovitskiy2021vit}. This control is not identical to the white-box CRATE architecture: it retains the standard ViT patch embedding, self-attention, normalization, and classifier, while replacing every MLP branch with a simple ReLU
activation. To compensate for the reduced capacity, we initialize from pretrained weights and use hard-label knowledge distillation following DeiT \citep{touvron2021deit}. A pretrained ViT-Tiny teacher provides its highest-confidence class as a pseudo-label. The student minimizes an equal-weighted sum of cross-entropy losses against the ground-truth label and the teacher pseudo-label.

\begin{table}[t]
\centering
\begin{tabular}{llcc}
\toprule
Scale & Model & Params & CIFAR-10 \\
\midrule
\multirow{2}{*}{Tiny}  & VIT-T     & 5.3M  & 92.19 \\
                       & VIT-T-AoT & 1.8M  & 91.37 \\
\midrule
\multirow{2}{*}{Small} & VIT-S     & 21.4M & 93.58 \\
                       & VIT-S-AoT & 7.2M  & 93.67 \\
\bottomrule
\end{tabular}
\caption{CIFAR-10 validation of standard ViTs and their MLP-free variants using hard-label knowledge distillation.}
\label{tab:vit}
\end{table}

Table~\ref{tab:vit} shows that removing all MLP blocks reduces the parameter count by approximately $66\%$ at both model scales. For the Tiny model, the parameter count decreases from 5.3M to 1.8M, while CIFAR-10 accuracy drops only slightly from $92.19\%$ to $91.37\%$, corresponding to a gap of $0.82$ percentage points. For the Small model, the parameter count is reduced from 21.4M to 7.2M, yet the accuracy increases marginally from $93.58\%$ to $93.67\%$. Although these results do not establish that MLP blocks are universally redundant, they indicate that optimization-derived proximal updates can substantially improve parameter efficiency while preserving competitive performance, providing a principled route toward MLP-free black-box Transformers.

\section{Conclusion}
We developed an attention-only white-box Transformer by jointly reconsidering self-supervised learning and optimization-based architecture derivation. Our key theoretical observation is that LeJEPA's SIGReg and the global expansion term $R(Z)$ share a consistent optimum under the isotropic covariance constraint. This connection allows the expansion term to be incorporated into the self-supervised training loss, while the remaining compression and sparsity terms are used to derive the network architecture. By applying ADMM and unrolling the resulting updates, we obtain an attention-only architecture in which MSSA implements the compression step and ReLU realizes the sparsity update. The resulting \ourmodel requires neither the ISTA dictionary nor MLP layers, establishing a unified connection among the white-box objective, the self-supervised loss, and the forward architecture.

Experiments under identical LeJEPA pretraining show that \ourmodel maintains performance comparable to CRATE on CIFAR-10 and CIFAR-100 while reducing the parameter count by approximately $31\%$. After ImageNet-1K self-supervised pretraining and full CIFAR fine-tuning, \ourmodel achieves $95.25\%$ on CIFAR-10 and $80.26\%$ on CIFAR-100, compared with CRATE's $96.01\%$ and $81.04\%$, while reducing the parameter count from 21.92M to 14.84M. At nearly identical parameter counts, \ourmodel also consistently outperforms Wang et al.'s AoT across different model scales, demonstrating that the optimization principle underlying an attention-only architecture is important for its effectiveness. Experiments on standard ViTs further show that the ReLU update can replace MLP blocks while preserving competitive performance, suggesting that the proposed formulation is not limited to white-box architectures. The theoretical analysis and empirical results both substantiate the partial redundancy of MLP modules in white-box Transformer architectures under self-supervised training, while providing further motivation for the development of attention-only Transformer architectures.

\bibliography{references}

@inproceedings{vaswani2017attention,
  title     = {Attention Is All You Need},
  author    = {Vaswani, Ashish and Shazeer, Noam and Parmar, Niki and Uszkoreit, Jakob and Jones, Llion and Gomez, Aidan N. and Kaiser, Lukasz and Polosukhin, Illia},
  booktitle = {Advances in Neural Information Processing Systems},
  volume    = {30},
  year      = {2017}
}

@inproceedings{yu2020mcr2,
  title     = {Learning Diverse and Discriminative Representations via the Principle of Maximal Coding Rate Reduction},
  author    = {Yu, Yaodong and Chan, Kwan Ho Ryan and You, Chong and Song, Chaobing and Ma, Yi},
  booktitle = {Advances in Neural Information Processing Systems},
  volume    = {33},
  pages     = {9422--9434},
  year      = {2020}
}

@inproceedings{chen2020simclr,
  title     = {A Simple Framework for Contrastive Learning of Visual Representations},
  author    = {Chen, Ting and Kornblith, Simon and Norouzi, Mohammad and Hinton, Geoffrey},
  booktitle = {Proceedings of the 37th International Conference on Machine Learning},
  series    = {Proceedings of Machine Learning Research},
  volume    = {119},
  pages     = {1597--1607},
  year      = {2020},
  publisher = {PMLR}
}

@inproceedings{grill2020byol,
  title     = {Bootstrap Your Own Latent: A New Approach to Self-Supervised Learning},
  author    = {Grill, Jean-Bastien and Strub, Florian and Altch{\'e}, Florent and Tallec, Corentin and Richemond, Pierre H. and Buchatskaya, Elena and Doersch, Carl and Avila Pires, Bernardo and Guo, Zhaohan Daniel and Azar, Mohammad Gheshlaghi and Piot, Bilal and Kavukcuoglu, Koray and Munos, R{\'e}mi and Valko, Michal},
  booktitle = {Advances in Neural Information Processing Systems},
  volume    = {33},
  pages     = {21271--21284},
  year      = {2020}
}

@inproceedings{caron2021dino,
  title     = {Emerging Properties in Self-Supervised Vision Transformers},
  author    = {Caron, Mathilde and Touvron, Hugo and Misra, Ishan and J{\'e}gou, Herv{\'e} and Mairal, Julien and Bojanowski, Piotr and Joulin, Armand},
  booktitle = {Proceedings of the IEEE/CVF International Conference on Computer Vision},
  pages     = {9650--9660},
  year      = {2021}
}

@inproceedings{he2022mae,
  title     = {Masked Autoencoders Are Scalable Vision Learners},
  author    = {He, Kaiming and Chen, Xinlei and Xie, Saining and Li, Yanghao and Doll{\'a}r, Piotr and Girshick, Ross},
  booktitle = {Proceedings of the IEEE/CVF Conference on Computer Vision and Pattern Recognition},
  pages     = {16000--16009},
  year      = {2022}
}

@inproceedings{assran2023ijepa,
  title     = {Self-Supervised Learning from Images with a Joint-Embedding Predictive Architecture},
  author    = {Assran, Mahmoud and Duval, Quentin and Misra, Ishan and Bojanowski, Piotr and Vincent, Pascal and Rabbat, Michael and Ballas, Nicolas},
  booktitle = {Proceedings of the IEEE/CVF Conference on Computer Vision and Pattern Recognition},
  pages     = {15619--15629},
  year      = {2023}
}

@inproceedings{chen2021mocov3,
  title     = {An Empirical Study of Training Self-Supervised Vision Transformers},
  author    = {Chen, Xinlei and Xie, Saining and He, Kaiming},
  booktitle = {Proceedings of the IEEE/CVF International Conference on Computer Vision},
  pages     = {9640--9649},
  year      = {2021}
}

@inproceedings{dong2021attention,
  title     = {Attention Is Not All You Need: Pure Attention Loses Rank Doubly Exponentially with Depth},
  author    = {Dong, Yihe and Cordonnier, Jean-Baptiste and Loukas, Andreas},
  booktitle = {Proceedings of the 38th International Conference on Machine Learning},
  series    = {Proceedings of Machine Learning Research},
  volume    = {139},
  pages     = {2793--2803},
  year      = {2021},
  publisher = {PMLR}
}

@inproceedings{he2024simplifying,
  title     = {Simplifying Transformer Blocks},
  author    = {He, Bobby and Hofmann, Thomas},
  booktitle = {International Conference on Learning Representations},
  year      = {2024}
}

@inproceedings{gregor2010lista,
  title     = {Learning Fast Approximations of Sparse Coding},
  author    = {Gregor, Karol and LeCun, Yann},
  booktitle = {Proceedings of the 27th International Conference on Machine Learning},
  pages     = {399--406},
  year      = {2010}
}

@article{beck2009fista,
  title   = {A Fast Iterative Shrinkage-Thresholding Algorithm for Linear Inverse Problems},
  author  = {Beck, Amir and Teboulle, Marc},
  journal = {SIAM Journal on Imaging Sciences},
  volume  = {2},
  number  = {1},
  pages   = {183--202},
  year    = {2009},
  doi     = {10.1137/080716542}
}

@article{chen1998basis,
  title   = {Atomic Decomposition by Basis Pursuit},
  author  = {Chen, Scott Shaobing and Donoho, David L. and Saunders, Michael A.},
  journal = {SIAM Journal on Scientific Computing},
  volume  = {20},
  number  = {1},
  pages   = {33--61},
  year    = {1998},
  doi     = {10.1137/S1064827596304010}
}

@article{parikh2014proximal,
  title   = {Proximal Algorithms},
  author  = {Parikh, Neal and Boyd, Stephen},
  journal = {Foundations and Trends in Optimization},
  volume  = {1},
  number  = {3},
  pages   = {127--239},
  year    = {2014},
  doi     = {10.1561/2400000003}
}

@article{daubechies2004iterative,
  title   = {An Iterative Thresholding Algorithm for Linear Inverse Problems with a Sparsity Constraint},
  author  = {Daubechies, Ingrid and Defrise, Michel and De Mol, Christine},
  journal = {Communications on Pure and Applied Mathematics},
  volume  = {57},
  number  = {11},
  pages   = {1413--1457},
  year    = {2004},
  doi     = {10.1002/cpa.20042}
}

@article{ouyang2015linearized,
  title   = {An Accelerated Linearized Alternating Direction Method of Multipliers},
  author  = {Ouyang, Yuyuan and Chen, Yunmei and Lan, Guanghui and Pasiliao, Eduardo, Jr.},
  journal = {SIAM Journal on Imaging Sciences},
  volume  = {8},
  number  = {1},
  pages   = {644--681},
  year    = {2015},
  doi     = {10.1137/14095697X}
}

@article{eckstein1992douglas,
  title   = {On the Douglas--Rachford Splitting Method and the Proximal Point Algorithm for Maximal Monotone Operators},
  author  = {Eckstein, Jonathan and Bertsekas, Dimitri P.},
  journal = {Mathematical Programming},
  volume  = {55},
  number  = {1--3},
  pages   = {293--318},
  year    = {1992},
  doi     = {10.1007/BF01581204}
}

@article{yu2023crate,
  title={White-Box Transformers via Sparse Rate Reduction},
  author={Yu, Yaodong and Buchanan, Sam and Pai, Druv and Chu, Tianzhe and Wu, Ziyang and Tong, Shengbang and Haeffele, Benjamin D. and Ma, Yi},
  journal={Advances in Neural Information Processing Systems},
  volume={36},
  year={2023}
}

@article{yang2024cratealpha,
  title={Scaling White-Box Transformers for Vision},
  author={Yang, Jinrui and Li, Xianhang and Pai, Druv and Zhou, Yuyin and Ma, Yi and Yu, Yaodong and Xie, Cihang},
  journal={arXiv preprint arXiv:2405.20299},
  year={2024}
}

@article{wu2024tost,
  title={Token Statistics Transformer: Linear-Time Attention via Variational Rate Reduction},
  author={Wu, Ziyang and Ding, Tianjiao and Lu, Yifu and Pai, Druv and Zhang, Jingyuan and Wang, Weida and Yu, Yaodong and Ma, Yi and Haeffele, Benjamin D.},
  journal={arXiv preprint arXiv:2412.17810},
  year={2024}
}

@article{yu2023segmentation,
  title={Emergence of Segmentation with Minimalistic White-Box Transformers},
  author={Yu, Yaodong and Pai, Druv and Buchanan, Sam and Wu, Ziyang and Tong, Shengbang and Haeffele, Benjamin D. and Ma, Yi},
  journal={arXiv preprint arXiv:2308.16271},
  year={2023}
}

@inproceedings{pai2024cratemae,
  title={Masked Completion via Structured Diffusion with White-Box Transformers},
  author={Pai, Druv and Wu, Ziyang and Buchanan, Sam and Yu, Yaodong and Ma, Yi},
  booktitle={International Conference on Learning Representations},
  year={2024}
}

@article{tong2023empssl,
  title={{EMP-SSL}: Towards Self-Supervised Learning in One Training Epoch},
  author={Tong, Shengbang and Chen, Yubei and Ma, Yi and LeCun, Yann},
  journal={arXiv preprint arXiv:2304.03977},
  year={2023}
}

@article{balestriero2025lejepa,
  title={{LeJEPA}: Provable and Scalable Self-Supervised Learning Without the Heuristics},
  author={Balestriero, Randall and LeCun, Yann},
  journal={arXiv preprint arXiv:2511.08544},
  year={2025}
}

@inproceedings{wang2025aot,
  title={Attention-Only Transformers via Unrolled Subspace Denoising},
  author={Wang, Peng and Lu, Yifu and Yu, Yaodong and Pai, Druv and Qu, Qing and Ma, Yi},
  booktitle={Proceedings of the 42nd International Conference on Machine Learning},
  series={Proceedings of Machine Learning Research},
  volume={267},
  pages={63840--63859},
  year={2025}
}

@article{boyd2011admm,
  title={Distributed Optimization and Statistical Learning via the Alternating Direction Method of Multipliers},
  author={Boyd, Stephen and Parikh, Neal and Chu, Eric and Peleato, Borja and Eckstein, Jonathan},
  journal={Foundations and Trends in Machine Learning},
  volume={3},
  number={1},
  pages={1--122},
  year={2011}
}

@techreport{krizhevsky2009cifar,
  title={Learning Multiple Layers of Features from Tiny Images},
  author={Krizhevsky, Alex},
  institution={University of Toronto},
  year={2009}
}

@inproceedings{deng2009imagenet,
  title={ImageNet: A Large-Scale Hierarchical Image Database},
  author={Deng, Jia and Dong, Wei and Socher, Richard and Li, Li-Jia and Li, Kai and Fei-Fei, Li},
  booktitle={IEEE Conference on Computer Vision and Pattern Recognition},
  pages={248--255},
  year={2009}
}

@inproceedings{loshchilov2019adamw,
  title={Decoupled Weight Decay Regularization},
  author={Loshchilov, Ilya and Hutter, Frank},
  booktitle={International Conference on Learning Representations},
  year={2019}
}

@inproceedings{dosovitskiy2021vit,
  title={An Image Is Worth 16x16 Words: Transformers for Image Recognition at Scale},
  author={Dosovitskiy, Alexey and Beyer, Lucas and Kolesnikov, Alexander and Weissenborn, Dirk and Zhai, Xiaohua and Unterthiner, Thomas and Dehghani, Mostafa and Minderer, Matthias and Heigold, Georg and Gelly, Sylvain and Uszkoreit, Jakob and Houlsby, Neil},
  booktitle={International Conference on Learning Representations},
  year={2021}
}

@inproceedings{touvron2021deit,
  title={Training Data-Efficient Image Transformers and Distillation through Attention},
  author={Touvron, Hugo and Cord, Matthieu and Douze, Matthijs and Massa, Francisco and Sablayrolles, Alexandre and J{\'e}gou, Herv{\'e}},
  booktitle={International Conference on Machine Learning},
  pages={10347--10357},
  year={2021}
}

@article{doi:10.1137/090777761,
author = {Yang, Junfeng and Zhang, Yin},
title = {Alternating Direction Algorithms for {$\ell_1$}-Problems in Compressive Sensing},
journal = {SIAM Journal on Scientific Computing},
volume = {33},
number = {1},
pages = {250-278},
year = {2011},
}

\clearpage
\appendix
\setcounter{secnumdepth}{2}
\numberwithin{equation}{section}

\section*{Appendix}

\section{Proof of Theorem 1}
\label{app:theorem1}

This appendix provides the complete proof of Theorem~1 stated in the main
text, together with a global extension of the correspondence that it
establishes between the two expansion objectives.

To avoid collision with the notation of the main text
($\gamma$: compression-gradient coefficient;
$\eta$: unrolled step size; and
$\lambda$: sparsity weight), we define
\begin{equation*}
  c=\frac{d}{\epsilon^2}>0,
  \qquad
  \theta=\frac{c}{1+c}\in(0,1),
  \qquad
  A=\Sigma-I_d.
  \label{eq:sup_notation}
\end{equation*}
Let $z\in\mathbb R^d$ denote the population embedding on which SIGReg is
evaluated, written as $Z'$ in the main text, with
\begin{equation*}
  \mu=\mathbb E[z]=0,
  \qquad
  \Sigma=\mathbb E\bigl[zz^\top\bigr].
  \label{eq:sup_covariance}
\end{equation*}
Centering and the fixed normalized-energy condition of Theorem~1 are imposed
by the normalization preceding the regularizer. Throughout this appendix the
two objectives are compared on
\begin{equation*}
  \mathcal S
  =
  \bigl\{
    \Sigma\in\mathbb S^{d}_{+}
    :
    \tr(\Sigma)=d
  \bigr\}.
  \label{eq:sup_normalization}
\end{equation*}
Since $A=\Sigma-I_d$, every $\Sigma\in\mathcal S$ satisfies $\tr(A)=0$, a
fact used repeatedly below. On the unconstrained cone $\mathbb S^{d}_{+}$ one
has $R(\alpha\Sigma)\to\infty$ as $\alpha\to\infty$ for every $\Sigma\neq0$,
so $R$ attains no maximum there. Both optima exist on $\mathcal S$, by
Eqs.~\eqref{eq:sup_sigreg_argmin} and \eqref{eq:sup_rate_argmax} below.

The coding rate, expressed as a function of the population covariance, is
\begin{equation}
  R(\Sigma)
  =
  \frac{1}{2}
  \log\det(I_d+c\Sigma),
  \label{eq:sup_rate}
\end{equation}
which is the population form of the empirical objective: writing
$\hat\Sigma=\frac1nZZ^\top$ for the sample covariance, one has
$\frac12\log\det\bigl(I_d+\frac{d}{n\epsilon^2}ZZ^\top\bigr)
 =\frac12\log\det(I_d+c\hat\Sigma)$
with the same constant $c=d/\epsilon^2$.

\subsection{Covariance Component of SIGReg}

For $a\sim\operatorname{Unif}(\mathbb S^{d-1})$, let
\begin{equation*}
  u_a=a^\top z,
  \qquad
  \varphi_a(t)=\mathbb E[e^{\mathrm{i}tu_a}],
  \qquad
  s_a=a^\top(\Sigma-I_d)a,
  \label{eq:sup_projection}
\end{equation*}
where the projections are taken unstandardized; for $a^\top\Sigma a>0$ one
has $1+s_a=a^\top\Sigma a$, and a per-direction rescaling of $u_a$ replaces
$\varphi_a$ by the characteristic function of $u_a/\sqrt{1+s_a}$, in which
$s_a$ does not appear. The Epps--Pulley statistic used by SIGReg is
\begin{equation}
  \SIGReg
  =
  \mathbb E_a
  \int_{\mathbb R}
  \left|
    \varphi_a(t)-e^{-t^2/2}
  \right|^2
  w(t)\,dt,
  \label{eq:sup_ep}
\end{equation}
with weight function $w\geq0$ positive Lebesgue-almost everywhere and
satisfying $\int_{\mathbb R}(1+t^{8})w(t)\,dt<\infty$. The moment condition
gives $C_w<\infty$ in Eq.~\eqref{eq:sup_sigreg2}; positivity almost
everywhere is used in the first part of the proof.

Let $\kappa_m(a)$ denote the $m$-th cumulant of $u_a$ and set
\begin{equation*}
  \Delta_{{>}2}
  =
  \sup_a\sup_{m\geq3}\lvert\kappa_m(a)\rvert,
  \qquad
  \Delta
  =
  \max\bigl\{
    \textstyle\sup_a\lvert s_a\rvert,\ \Delta_{{>}2}
  \bigr\}.
  \label{eq:sup_delta}
\end{equation*}
Around the standard isotropic Gaussian target, and assuming
$\Delta_{{>}2}<\infty$ so that the series may be integrated against $w$ term
by term, the projected characteristic function admits the local expansion
\begin{equation}
\begin{aligned}
&\varphi_a(t)-e^{-t^2/2}\\
  &=
  e^{-t^2/2}
  \left[
    -\frac{s_a t^2}{2}
    +
    \sum_{m\geq3}
    \kappa_m(a)\frac{(\mathrm{i}t)^m}{m!}
    +
    O(\Delta^2)
  \right].
\end{aligned}
  \label{eq:sup_cf_expansion}
\end{equation}

Squaring Eq.~\eqref{eq:sup_cf_expansion}, the terms free of any cumulant of
order three or higher reduce to
\begin{equation}
  e^{-t^2}\frac{s_a^2t^4}{4},
  \label{eq:sup_covariance_integrand}
\end{equation}
while every remaining term carries at least one factor $\kappa_m(a)$ with
$m\geq3$. We take this cumulant order as the grouping criterion, collecting
the former into $\mathcal L_{\mathrm{SIGReg}}^{(2)}$ and the latter into the
term $\mathcal L_{>2}$ used in the main text. The grouping is not an
orthogonal decomposition: the cross terms $s_a\kappa_m(a)$ enter
$\mathcal L_{>2}$ at the same order in $\Delta$ as $s_a^2$ enters
$\mathcal L_{\mathrm{SIGReg}}^{(2)}$. On the Gaussian slice
\begin{equation}
  \mathcal G
  =
  \bigl\{
    \mathcal N(0,\Sigma)
    :
    \Sigma\in\mathcal S
  \bigr\},
  \qquad
  \kappa_m(a)\equiv0
  \quad (m\geq3),
  \label{eq:sup_gaussian_slice}
\end{equation}
one has $\mathcal L_{>2}=0$ and the two components separate exactly;
elsewhere they separate up to $O(\Delta_{{>}2})$. The qualifier ``at the
covariance level'' in Theorem~1 refers to this restriction.

Integrating Eq.~\eqref{eq:sup_covariance_integrand} against $w(t)$ and
averaging over $a$ gives
\begin{equation}
\begin{aligned}
   & \mathcal L_{\mathrm{SIGReg}}^{(2)}(\Sigma)
  =
  C_w\,
  \mathbb E_{a\sim\operatorname{Unif}(\mathbb S^{d-1})}
  \left[
    (a^\top A a)^2
  \right],\\
  &C_w
  =
  \frac{1}{4}
  \int_{\mathbb R}
  t^4e^{-t^2}w(t)\,dt
  \in(0,\infty).
\end{aligned}
  \label{eq:sup_sigreg2}
\end{equation}

\subsection{Proof of the Shared Optimum and Local Correspondence}

\begin{proof}[Proof of Theorem~1]
The integrand of Eq.~\eqref{eq:sup_ep} is nonnegative, so $\SIGReg=0$ forces
$\varphi_a(t)=e^{-t^2/2}$ for almost every $t$ and almost every
$a\in\mathbb S^{d-1}$. Since $w>0$ almost everywhere and both sides are
continuous in $t$, the identity holds for every $t$; continuity of
$a\mapsto\varphi_a(t)$ extends it to every $a$. Every one-dimensional
projection of $z$ is then standard normal, and the Cram\'er--Wold device
gives
\begin{equation}
  \SIGReg=0
  \iff
  z\sim\mathcal N(0,I_d)
  \implies
  \Sigma=I_d,
  \label{eq:sup_cramer_wold}
\end{equation}
the minimum being attained on $\mathcal S$ because $\mathcal N(0,I_d)$
satisfies $\tr(\Sigma)=d$. Equation~\eqref{eq:sup_cramer_wold} characterizes
the minimizer as a distribution; $R$ depends on $\Sigma$ alone.

Restricted to the covariance level, the same conclusion follows in closed
form. For $a\sim\operatorname{Unif}(\mathbb S^{d-1})$, the fourth-order
spherical moments satisfy
\begin{equation}
  \mathbb E[a_i a_j a_k a_l]
  =
  \frac{
    \delta_{ij}\delta_{kl}
    +
    \delta_{ik}\delta_{jl}
    +
    \delta_{il}\delta_{jk}
  }{d(d+2)}.
  \label{eq:sup_spherical_moment}
\end{equation}

Using Eq.~\eqref{eq:sup_spherical_moment},
\begin{align}
  \mathbb E_a[(a^\top A a)^2]
  &=
  \sum_{i,j,k,l}
  A_{ij}A_{kl}
  \mathbb E[a_i a_j a_k a_l]
  \nonumber\\
  &=
  \frac{
    \tr(A)^2
    +
    \tr(A^\top A)
    +
    \tr(A^2)
  }{d(d+2)}.
  \label{eq:sup_spherical_expansion}
\end{align}

Since $A$ is symmetric,
\begin{equation}
  \tr(A^\top A)
  =
  \tr(A^2)
  =
  \lVert A\rVert_F^2.
  \label{eq:sup_symmetric_A}
\end{equation}

Together with $\tr(A)=0$,
\begin{equation}
  \mathbb E_a[(a^\top A a)^2]
  =
  \frac{2}{d(d+2)}
  \lVert A\rVert_F^2.
  \label{eq:sup_spherical_identity}
\end{equation}

Substituting Eq.~\eqref{eq:sup_spherical_identity} into
Eq.~\eqref{eq:sup_sigreg2},
\begin{equation}
  \mathcal L_{\mathrm{SIGReg}}^{(2)}(\Sigma)
  =
  \frac{2C_w}{d(d+2)}
  \lVert\Sigma-I_d\rVert_F^2.
  \label{eq:sup_sigreg_frobenius}
\end{equation}

Since $C_w>0$, the right-hand side is nonnegative and vanishes if and only if
$\Sigma=I_d$, in agreement with Eq.~\eqref{eq:sup_cramer_wold}. Hence
\begin{equation}
  \arg\min_{\Sigma\in\mathcal S}
  \mathcal L_{\mathrm{SIGReg}}^{(2)}(\Sigma)
  =
  \{I_d\}.
  \label{eq:sup_sigreg_argmin}
\end{equation}

We next determine the maximizer of the coding rate. Let
$\nu_1,\ldots,\nu_d$ denote the eigenvalues of $\Sigma$. Since
$\Sigma\succeq0$ and $\tr(\Sigma)=d$,
\begin{equation}
  \nu_i\geq0,
  \qquad
  \sum_{i=1}^d\nu_i=d,
  \label{eq:sup_eigenvalue_constraint}
\end{equation}
and Eq.~\eqref{eq:sup_rate} becomes
\begin{equation}
  R(\Sigma)
  =
  \frac{1}{2}
  \sum_{i=1}^d
  \log(1+c\nu_i).
  \label{eq:sup_rate_eigenvalues}
\end{equation}

The function $g(x)=\log(1+cx)$ is strictly concave on $[0,\infty)$ since
\begin{equation}
  g''(x)
  =
  -\frac{c^2}{(1+cx)^2}
  <0.
  \label{eq:sup_concavity}
\end{equation}

Jensen's inequality applied to Eq.~\eqref{eq:sup_rate_eigenvalues} gives
\begin{align}
  R(\Sigma)
  &\leq
  \frac{d}{2}
  \log
  \left(
    1+
    \frac{c}{d}
    \sum_{i=1}^d\nu_i
  \right)
  \nonumber\\
  &=
  \frac{d}{2}\log(1+c)
  =
  R(I_d).
  \label{eq:sup_jensen}
\end{align}

By strict concavity of $g$, equality holds if and only if
$\nu_1=\cdots=\nu_d$, and the trace constraint forces
$\nu_1=\cdots=\nu_d=1$, i.e.\ $\Sigma=I_d$. Hence
\begin{equation}
  \arg\max_{\Sigma\in\mathcal S}
  R(\Sigma)
  =
  \{I_d\}.
  \label{eq:sup_rate_argmax}
\end{equation}

Equations~\eqref{eq:sup_sigreg_argmin} and \eqref{eq:sup_rate_argmax} give
the shared-optimum statement in Theorem 1.

It remains to establish the local correspondence around $\Sigma=I_d$.
Substituting $\Sigma=I_d+A$ into Eq.~\eqref{eq:sup_rate},
\begin{align}
  R(I_d+A)
  &=
  \frac{1}{2}
  \log\det\left(
    I_d+c(I_d+A)
  \right)
  \nonumber\\
  &=
  \frac{1}{2}
  \log\det\left(
    (1+c)I_d+cA
  \right)
  \nonumber\\
  &=
  \frac{d}{2}\log(1+c)
  +
  \frac{1}{2}
  \log\det(I_d+\theta A).
  \label{eq:sup_rate_factorization}
\end{align}

On $\mathcal S$ one has $\lVert A\rVert_2\leq d-1$, so the condition
$\lVert A\rVert_2<1/\theta$ defines a proper subset of $\mathcal S$
containing $I_d$. On that subset the log-determinant expansion gives
\begin{equation}
  \log\det(I_d+\theta A)
  =
  \theta\tr(A)
  -
  \frac{\theta^2}{2}\tr(A^2)
  +
  O\bigl(\lVert A\rVert_F^3\bigr).
  \label{eq:sup_logdet_expansion}
\end{equation}

Using $\tr(A)=0$ and $\tr(A^2)=\lVert A\rVert_F^2$,
\begin{equation}
  \log\det(I_d+\theta A)
  =
  -\frac{\theta^2}{2}
  \lVert A\rVert_F^2
  +
  O\bigl(\lVert A\rVert_F^3\bigr).
  \label{eq:sup_logdet_reduced}
\end{equation}

With $R(I_d)=\frac{d}{2}\log(1+c)$, Eqs.~\eqref{eq:sup_rate_factorization}
and \eqref{eq:sup_logdet_reduced} give
\begin{equation}
  R(I_d)-R(\Sigma)
  =
  \frac{\theta^2}{4}
  \lVert\Sigma-I_d\rVert_F^2
  +
  O\left(
    \lVert\Sigma-I_d\rVert_F^3
  \right).
  \label{eq:sup_rate_gap}
\end{equation}

Equation~\eqref{eq:sup_sigreg_frobenius} shows that
$\mathcal L_{\mathrm{SIGReg}}^{(2)}$ is a positive multiple of
$\lVert\Sigma-I_d\rVert_F^2$, and Eq.~\eqref{eq:sup_rate_gap} shows that the
leading term of the coding-rate gap is likewise a positive multiple of it, so
that
\begin{equation}
  \lim_{\Sigma\to I_d}
  \frac{
    \mathcal L_{\mathrm{SIGReg}}^{(2)}(\Sigma)
  }{
    R(I_d)-R(\Sigma)
  }
  =
  \frac{8C_w}{\theta^2 d(d+2)}
  \in(0,\infty).
  \label{eq:sup_proportionality}
\end{equation}

Equations~\eqref{eq:sup_sigreg_argmin}, \eqref{eq:sup_rate_argmax} and
\eqref{eq:sup_proportionality} prove both statements of Theorem~1.
\end{proof}

\subsection{Global Equivalence of the Two Expansion Terms}
\label{app:global_equivalence}

Theorem~1 gives a shared optimum and a leading-order correspondence at
$\Sigma=I_d$. The following proposition extends the correspondence to all of
$\mathcal S$.

\begin{proposition}
\label{prop:global_equivalence}
There exist constants $0<k_1\leq k_2<\infty$, depending only on $d$, $c$ and
$w$, such that
\begin{equation}
  k_1\bigl(R(I_d)-R(\Sigma)\bigr)
  \leq
  \mathcal L_{\mathrm{SIGReg}}^{(2)}(\Sigma)
  \leq
  k_2\bigl(R(I_d)-R(\Sigma)\bigr)
  \qquad
  \label{eq:sup_global_equivalence}
\end{equation}
for all $\Sigma\in\mathcal S$.

Consequently, for every $\tau>0$,
\begin{equation}
\begin{aligned}
    &\bigl\{
    \mathcal L_{\mathrm{SIGReg}}^{(2)}\leq k_1\tau
  \bigr\}\\
  &\subseteq
  \bigl\{
    R(I_d)-R\leq\tau
  \bigr\}\\
  &\subseteq
  \bigl\{
    \mathcal L_{\mathrm{SIGReg}}^{(2)}\leq k_2\tau
  \bigr\}.
\end{aligned}
  \label{eq:sup_sublevel}
\end{equation}

Moreover, let $\mathcal P(X)=X-\frac{\tr(X)}{d}I_d$ denote the orthogonal
projection onto the traceless symmetric matrices. Then, with
$A=\Sigma-I_d$ as in Eq.~\eqref{eq:sup_notation},
\begin{equation}
\begin{aligned}
      &\mathcal P\bigl[
    \nabla_\Sigma\mathcal L_{\mathrm{SIGReg}}^{(2)}
  \bigr]
  =
  \frac{4C_w}{d(d+2)}A,\\
    &\mathcal P\bigl[
    \nabla_\Sigma\bigl(R(I_d)-R(\Sigma)\bigr)
  \bigr]
  =
  \frac{\theta^2}{2}A
  +
  O\bigl(\lVert A\rVert_F^2\bigr),
\end{aligned}
  \label{eq:sup_gradients}
\end{equation}
so that their inner product equals
$\frac{2C_w\theta^2}{d(d+2)}\lVert A\rVert_F^2
 +O\bigl(\lVert A\rVert_F^3\bigr)$.
\end{proposition}

\begin{proof}
Define
\begin{equation*}
  \rho(\Sigma)
  =
  \frac{
    \mathcal L_{\mathrm{SIGReg}}^{(2)}(\Sigma)
  }{
    R(I_d)-R(\Sigma)
  },
  \qquad
  \Sigma\in\mathcal S\setminus\{I_d\}.
  \label{eq:sup_ratio}
\end{equation*}

The eigenvalues of any $\Sigma\in\mathcal S$ lie in $[0,d]$, so $\mathcal S$
is compact. 

By Eqs.~\eqref{eq:sup_sigreg_argmin} and
\eqref{eq:sup_rate_argmax}, the numerator and denominator of $\rho$ are
continuous and strictly positive on $\mathcal S\setminus\{I_d\}$. On the
subset
\[
\lVert A\rVert_2\leq1/2\theta,
\]
 the eigenvalues of $\theta A$ are
bounded away from $-1$, so the remainder in Eq.~\eqref{eq:sup_rate_gap} is
bounded by $C\lVert A\rVert_F^3$ with $C$ depending only on $\theta$, hence
uniformly in the direction $A/\lVert A\rVert_F$. 

Dividing
Eq.~\eqref{eq:sup_sigreg_frobenius} by Eq.~\eqref{eq:sup_rate_gap} therefore
gives Eq.~\eqref{eq:sup_proportionality} as a limit along every sequence
$\Sigma\to I_d$ in $\mathcal S$, and $\rho$ extends continuously to
$\mathcal S$ with value $8C_w/\bigl(\theta^2d(d+2)\bigr)$ at $I_d$. 

A positive continuous function on a compact set attains a positive minimum and
a finite maximum; taking $k_1=\min_{\mathcal S}\rho$ and
$k_2=\max_{\mathcal S}\rho$ gives Eq.~\eqref{eq:sup_global_equivalence}.
Equation~\eqref{eq:sup_sublevel} follows from the two inequalities in
Eq.~\eqref{eq:sup_global_equivalence} separately: if
$\mathcal L_{\mathrm{SIGReg}}^{(2)}(\Sigma)\leq k_1\tau$, the left inequality
gives $k_1\bigl(R(I_d)-R(\Sigma)\bigr)\leq k_1\tau$ and hence the first
inclusion; if $R(I_d)-R(\Sigma)\leq\tau$, the right inequality gives the
second.

For the gradients, Eq.~\eqref{eq:sup_sigreg_frobenius} gives
$\nabla_\Sigma\mathcal L_{\mathrm{SIGReg}}^{(2)}
 =\frac{4C_w}{d(d+2)}A$, which is already traceless by $\tr(A)=0$. From
Eq.~\eqref{eq:sup_rate_factorization},
$\nabla_\Sigma\bigl[R(I_d)-R(\Sigma)\bigr]
 =-\frac{\theta}{2}(I_d+\theta A)^{-1}$; expanding
$(I_d+\theta A)^{-1}=I_d-\theta A+O(\lVert A\rVert_F^2)$ and applying
$\mathcal P$, which annihilates the multiple of $I_d$, gives the second
identity in Eq.~\eqref{eq:sup_gradients}. The stated inner product follows
from $\langle A,A\rangle=\lVert A\rVert_F^2$.
\end{proof}

Equation~\eqref{eq:sup_global_equivalence} bounds each expansion objective by
a fixed positive multiple of the other at every point of $\mathcal S$, and
Eq.~\eqref{eq:sup_sublevel} nests their sublevel sets: driving either
quantity below a given threshold drives the other below a constant multiple
of it, without restriction to a neighborhood of $I_d$.
Equation~\eqref{eq:sup_gradients} gives the corresponding first-order
statement: on the constraint tangent space the two projected gradients are
positive multiples of the same matrix to leading order, so a descent step on
either has positive inner product with a descent step on the other. SIGReg
can thus be used in place of $R(Z)$ as the expansion term in the LeJEPA
objective.

By Eq.~\eqref{eq:sup_cramer_wold}, the full statistic retains the same
optimum while characterizing it at the level of the entire distribution: $R$
is the covariance-level projection of SIGReg. Assigning the expansion role to
SIGReg thus retains the constraint imposed by $R(Z)$ and adds the
higher-order constraint carried by $\mathcal L_{>2}$, which no function of
$\Sigma$ alone can express. Equation~\eqref{eq:sup_proportionality} fixes the
scale: the proportionality constant decays as $O(d^{-2})$ while the
coding-rate gap carries the dimension-free factor $\theta^2/4$, so converting
the weight of the expansion term by this ratio preserves its effective
strength.

\section{Pseudocode of the Proposed Algorithm}

Algorithm~\ref{alg:aot_admm_pretraining} summarizes one end-to-end LeJEPA
pretraining iteration with the proposed AoT-ADMM encoder. The three
coefficients $(a_\ell,b_\ell,c_\ell)$ of the $Z$-update are parameterized as
a softmax over trainable logits, so that they remain positive and sum to one
throughout training; the logits are initialized so that the softmax outputs
$(1-\eta\gamma-\eta\rho,\eta\gamma,\eta\rho)$, recovering the unrolled ADMM
step at initialization. To stabilize the magnitudes of the recurrent states,
both $Z^{\ell+1}$ and the updated dual state $W^{\ell+1}$ are subjected to
feature-wise RMS rescaling at each layer. The threshold $\tau=\lambda/\rho$
retains the proximal interpretation of the sparsity update.

\begin{algorithm}[tb]
\caption{LeJEPA Pretraining with the AoT-ADMM Encoder}
\label{alg:aot_admm_pretraining}
\textbf{Input}: Mini-batch $\mathcal{B}$; global and local augmentations;
encoder depth $L$\\
\textbf{Parameters}: Subspace bases $\{U_{[K]}^\ell\}_{\ell=0}^{L-1}$,
branch coefficients $\{a_\ell,b_\ell,c_\ell\}_{\ell=0}^{L-1}$,
projection head $g$, and SIGReg weight $\alpha_{\mathrm{sig}}$\\
\textbf{Output}: Updated encoder and projection-head parameters
\begin{algorithmic}[1]
\STATE Generate two global views and the prescribed local views for every
image in $\mathcal{B}$.
\FOR{each augmented view $x^{(v)}$}
    \STATE $Z^{0,(v)} \leftarrow
    \operatorname{PatchEmbed}(x^{(v)})+\operatorname{PosEmbed}$
    \STATE $V^{0,(v)} \leftarrow Z^{0,(v)}$;
    $W^{0,(v)} \leftarrow 0$
    \FOR{$\ell=0$ to $L-1$}
        \STATE $A^{\ell,(v)} \leftarrow
        \MSSA(Z^{\ell,(v)}\mid U_{[K]}^\ell)$
        \STATE $Z^{\ell+1,(v)} \leftarrow
        a_\ell Z^{\ell,(v)}
        +b_\ell A^{\ell,(v)}
        +c_\ell(V^{\ell,(v)}-W^{\ell,(v)})$
        \STATE $V^{\ell+1,(v)} \leftarrow
        \operatorname{ReLU}
        (Z^{\ell+1,(v)}+W^{\ell,(v)}-\tau)$
        \STATE $W^{\ell+1,(v)} \leftarrow
        W^{\ell,(v)}+Z^{\ell+1,(v)}-V^{\ell+1,(v)}$
    \ENDFOR
    \STATE $Y^{(v)} \leftarrow
    g\!\left(\operatorname{Pool}(V^{L,(v)})\right)$
\ENDFOR
\STATE Compute $\mathcal{L}_{\mathrm{pred}}$ from the local and global
embeddings and their mean global target.
\STATE Compute
$\mathcal{L}\leftarrow
\mathcal{L}_{\mathrm{pred}}+
\alpha_{\mathrm{sig}}\SIGReg$.
\STATE Update the encoder, layerwise subspace bases, branch coefficients, and
projection head by backpropagation while maintaining the coefficient
constraints.
\STATE \textbf{return} the updated parameters.
\end{algorithmic}
\end{algorithm}

The inner loop is exactly one unrolled ADMM iteration per network layer:
MSSA performs the compression update, ReLU implements the nonnegative
soft-thresholding step, and the dual state $W$ tracks the consistency residual.
The outer part of the algorithm places this encoder inside the same multi-crop
LeJEPA objective used for all reported self-supervised experiments.

\end{document}